\documentclass[a4paper,11pt]{article}

\usepackage[utf8]{inputenc}
\usepackage[T1]{fontenc}
\usepackage[english]{babel}

\usepackage{hyperref}
\usepackage{url}
\usepackage{booktabs}
\usepackage{amsfonts}
\usepackage{nicefrac}
\usepackage{microtype}

\usepackage[pdftex]{graphicx}

\usepackage{amsmath,amsthm}

\newtheorem{theorem}{Theorem}
\newtheorem{corollary}{Corollary}

\newcommand{\E}{\textnormal{E}}
\newcommand{\Prp}[1]{\Pr\!\left[{#1} \right]}
\newcommand{\Ep}[1]{\E\!\left[{#1} \right]}
\newcommand{\floor}[1]{\left\lfloor{#1}\right\rfloor}
\newcommand{\set}[1]{\left \{ #1 \right \}}
\newcommand{\eps}{\varepsilon}
\newcommand{\R}{\mathbb{R}}
\newcommand{\norm}[1]{\left \| #1 \right \|}
\newcommand{\abs}[1]{\left | #1 \right |}

\def\sgn{\operatorname{sgn}}
\usepackage[usenames,dvipsnames]{xcolor}
\newcommand{\mt}[1]{{#1}} %{\color{red} {#1}}}

\usepackage{siunitx}

\usepackage{todonotes}
\usepackage{xargs}
\newcommandx{\sdtodo}[2][1=]{\todo[linecolor=red,backgroundcolor=red!25,bordercolor=red,#1]{SD:#2}}

\usepackage{fullpage}

\usepackage{authblk}
\title{Practical Hash Functions for Similarity Estimation and Dimensionality
Reduction\thanks{Code for this paper is available at
\url{https://github.com/zera/Nips_MT}}}

\author[1,2]{Søren Dahlgaard}
\author[1,2]{Mathias Bæk Tejs Knudsen}
\author[1]{Mikkel Thorup}
\affil[1]{University of Copenhagen -- \texttt{mthorup@di.ku.dk}}
\affil[2]{SupWiz -- \texttt{[s.dahlgaard,m.knudsen]@supwiz.com}}
\date{}
\begin{document} 

\maketitle

\begin{abstract}
    Hashing is a basic tool for dimensionality reduction employed in
    several aspects of machine learning. However, the perfomance
      analysis is often carried out under the \emph{abstract}
      assumption that a truly random unit cost hash function is used,
      without concern for which \emph{concrete} hash function is
      employed. The concrete hash function may work fine on
    sufficiently random input. The question is if they can be
    \emph{trusted} in the real world where they may be faced with more structured 
    input.

    In this paper we focus on two prominent applications
    of hashing, namely similarity estimation with the one permutation
    hashing (OPH) scheme of Li et al.~[NIPS'12] and feature hashing (FH) of
    Weinberger et al. [ICML'09], both of which have found numerous
    applications, i.e. in approximate near-neighbour search with LSH and
    large-scale classification with SVM.
    
    We consider the recent mixed tabulation hash function of Dahlgaard
    et al. [FOCS'15] which was proved theoretically to perform like
    a truly random hash function in many applications, including the
    above OPH. Here we first show improved concentration bounds for FH with
    truly random hashing and then argue that mixed tabulation performs similar
    when the input vectors are not too dense.
    Our main contribution, however, is an experimental comparison of
    different hashing schemes when used inside \mt{FH, OPH, and LSH}.

    We find that mixed tabulation hashing is almost as fast as the
    classic multiply-mod-prime scheme $(ax+b) \bmod
    p$. Mutiply-mod-prime is guaranteed to work well on sufficiently
    random data, but here we demonstrate that in the above
    applications, it can lead to bias and poor concentration on both
    real-world and synthetic data. We also compare with the very popular
    MurmurHash3, which has no proven guarantees. Mixed
    tabulation and MurmurHash3 both perform similar to truly random hashing in
    our experiments. However, mixed tabulation was 40\% faster than
    MurmurHash3, and it has the proven guarantee of good performance
    (like fully random) on all possible input making it more reliable.

%    \mt{Nedenfor er, hvad i havde skrevet. Synes at det er meget uklart fra teksten, hvad der kommer fra
%    FOCS, og hvad der er nyt}
%    We consider the very fast and powerful mixed tabulation hash function of
%    Dahlgaard et al. [FOCS'15] demonstrating that it systematically outperforms
%    hash functions with faster evaluation times both in terms of bias and
%    concentration for these applications. This is first proved theoretically,
%    adding to Dahlgaard et al. who showed it for OPH, and then demonstrated
%    with experiments on both synthetic and real world data. \mt{Det er fusk at skrive ``Our experimental
%    evaluation also shows that mixed tabulation is fifteen times faster than
%    hash functions with the same theoretical guarantees.''. Der er jo intet
%specielt ved 20-independence. Hvis vi havde brugt 1000-independence havde det
%vaeret endnu bedre.}
\end{abstract}

\section{Introduction}\label{sec:intro}
Hashing is a standard technique for dimensionality reduction and is employed as
an underlying tool in several aspects of machine learning including
search~\cite{li12oneperm,Shrivastava14oneperm,Shrivastava14densify,Andoni14lsh},
classification~\cite{li11minhash,li12oneperm}, duplicate
detection~\cite{manku07duplicates}, computer vision and information retrieval
\cite{ShakhnarovichDI08}. The need for dimensionality reduction techniques
such as hashing is becoming further important due to the huge growth in
data sizes. As an example, already in 2010, Tong~\cite{Tong10} discussed data
sets with $10^{11}$ data points and $10^9$ features. Furthermore, when working
with text, data points are often stored as $w$-shingles (i.e. $w$ contiguous
words or bytes) with $w\ge 5$. This further increases the dimension from, say,
$10^5$ \emph{common} english words to $10^{5w}$.

Two particularly prominent applications are set similarity estimation as
initialized by the MinHash algorithm of Broder, et
al.~\cite{broder97onthe,broder97minwise} and feature hashing (FH) of Weinberger, et
al.~\cite{WeinbergerDLSA09}. Both applications have in common that they are
used as an underlying ingredient in many other applications.
While both MinHash and FH can be seen as hash functions mapping an entire set
or vector, they are perhaps better described as algorithms implemented using
what we will call \emph{basic hash functions}. A basic hash function $h$ maps a
given key to a hash value, and any such basic hash function, $h$, can be used
to implement Minhash, which maps a set of keys, $A$, to the smallest hash value
$\min_{a\in A} h(a)$. A similar case can be made for other locality-sensitive
hash functions such as SimHash~\cite{Charikar02}, One Permutation
Hashing (OPH)~\cite{li12oneperm,Shrivastava14oneperm,Shrivastava14densify}, and
cross-polytope hashing~\cite{AndoniILRS15,TerasawaT07,KennedyW16}, which are
all implemented using basic hash functions.

\subsection{Importance of understanding basic hash functions}
In this paper we analyze the basic hash functions needed for the applications
of similarity estimation and FH. This is important for two
reasons: 1) As mentioned in \cite{li12oneperm}, dimensionality reduction is
often a time bottle-neck and using a fast basic hash function to implement it
may improve running times significantly, and 2) the theoretical guarantees of
hashing schemes such as Minhash and FH rely crucially on the basic
hash functions used to implement it, and this is further propagated into
applications of these schemes such as approximate similarity search with the
seminal LSH framework of Indyk and Motwani~\cite{IndykM98}.

To fully appreciate this, consider LSH for
approximate similarity search implemented with MinHash. We know
from~\cite{IndykM98} that this structure obtains \emph{provably} sub-linear
query time and \emph{provably} sub-quadratic space, where the exponent depends
on the probability of hash collisions for ``similar'' and ``not-similar'' sets.
However, we also know that implementing MinHash with a poorly chosen hash
function leads to \emph{constant bias} in the estimation~\cite{PatrascuT16},
and this constant then appears in the \emph{exponent} of both the space and the
query time of the search structure leading to worse theoretical guarantees.

Choosing the right basic hash function is an often overlooked aspect, and many
authors simply state that any (universal) hash function ``is usually sufficient
in practice'' (see e.g.~\cite[page 3]{li12oneperm}). While this is indeed the
case most of the time (and provably if
the input has enough entropy \cite{mitzenmacher08hash}), many applications
rely on taking advantage of highly structured data to perform well
(such as classification or similarity search). In these cases a poorly chosen
hash function may lead to very systematic inconsistensies. Perhaps the most
famous example of this is hashing with linear probing which was deemed
very fast but unrealiable in practice until it was fully understood which hash
functions to employ (see \cite{thorup12kwise} for discussion and experiments).
Other papers (see e.g.~\cite{Shrivastava14oneperm,Shrivastava14densify}
suggest using very powerful machinery such as the seminal pseudorandom
generator of Nisan~\cite{nisan92spaceprg}. However, such a PRG does not
represent a hash function and implementing it as such would incur a huge
computational overhead.

Meanwhile, some papers do indeed consider which concrete hash functions to
use. In \cite{dahlgaard13nnbottomk} it was considered to use $2$-independent
hashing for bottom-$k$ sketches, which was proved in \cite{thorup13bottomk} to
work for this application. However, bottom-$k$ sketches do not work for
SVMs and LSH. Closer to our work, \cite{Li12bbitpractice} considered the use of
$2$-independent (and $4$-independent) hashing for large-scale classification
and online learning with $b$-bit minwise hashing. Their experiments indicate
that $2$-independent hashing often works, and they state that ``the simple and
highly efficient $2$-independent scheme may be sufficient in practice''.
However, no amount of experiments can show that this is the case for all input.
In fact, we demonstrate in this paper -- for the underlying FH and OPH --
that this is not the case, and that we cannot trust $2$-independent hashing to
work in general. \mt{As noted, \cite{Li12bbitpractice} used hashing for
similarity estimation in classification, but without considering the quality of
the underlying similarity estimation. Due to space restrictions, we do not
consider classification in this paper, but instead focus on the quality of the
underlying similarity estimation and dimensionality reduction sketches as well
as considering these sketches in LSH as the sole applicaton (see also the
discussion below).}

\subsection{Our contribution}
We analyze the very fast and powerful mixed tabulation scheme
of~\cite{DahlgaardKRT15} comparing it to some of the most popular and widely
employed hash functions. In \cite{DahlgaardKRT15} it was shown that
implementing OPH with mixed tabulation gives
concentration bounds ``essentially as good as truly random''. For feature
hashing, we first present new concentration bounds for the truly random case
improving on \cite{WeinbergerDLSA09,DasguptaKS10}. We then argue that mixed
tabulation gives essentially as good concentration bounds in the case where the
input vectors are not too dense, which is a very common case for applying
feature hashing.

Experimentally, we demonstrate that mixed tabulation is almost as fast as the
classic multiply-mod-prime hashing scheme. This classic scheme is guaranteed to
work well for the considered applications when the data is sufficiently random,
but we demonstrate that bias and poor concentration can occur on both
synthetic and real-world data. \mt{We verify on the same experiments that mixed
tabulation has the desired strong concentration, confirming the theory.}
We also find that
mixed tabulation is roughly 40\% faster than the very popular MurmurHash3 and
CityHash. In our experiments these hash functions perform similar to mixed
tabulation in terms of concentration. They do, however, not have the same
theoretical guarantees making them harder to trust.
We also consider different basic hash functions for implementing LSH with
OPH. We demonstrate that the bias and poor concentration of the simpler
hash functions for OPH translates into poor concentration for e.g. the recall
and number of retrieved data points
of the corresponding LSH search structure. Again, we observe that this is not
the case for mixed tabulation, which systematically out-performs the faster
hash functions. We note that \cite{Li12bbitpractice} suggests that
$2$-independent hashing only has problems with dense data sets, but both the
real-world and synthetic data considered in this work are sparse or, in the
case of synthetic data, can be generalized to arbitrarily sparse data. While we
do not consider $b$-bit hashing as in \cite{Li12bbitpractice}, we note that
applying the $b$-bit trick to our experiments would only introduce a bias from
false positives for all basic hash functions and leave the conclusion the
same.

It is important to note that our results do not imply that standard hashing
techniques (i.e. multiply-mod prime) never work. Rather, they show that there
does exist practical scenarios where the theoretical guarantees matter, making
mixed tabulation more consistent. We believe that the very fast evaluation time
and consistency of mixed tabulation makes it the best choice for the
applications considered in this paper.

\section{Preliminaries}
As mentioned we focus on similarity estimation and feature hashing. Here we
briefly describe the methods used. We let $[m] = \{0,\ldots, m-1\}$, for some
integer $m$, denote the output range of the hash functions considered.

\subsection{Similarity estimation}
In similarity estimation we are given two sets, $A$ and $B$ belonging to some
universe $U$ and are tasked with estimating the Jaccard similarity $J(A,B) =
|A\cap B|/|A\cup B|$. As mentioned earlier, this can be solved using $k$
independent repetitions of the MinHash algorithm, however this requires
$O(k\cdot |A|)$ running time. In this paper we instead use the faster OPH of
Li et al.~\cite{li12oneperm} with the densification scheme of Shrivastava and
Li~\cite{Shrivastava14densify}. This scheme works as follows: Let $k$ be a
parameter with $k$ being a divisor of $m$,
%such that $k$ is a power of two.
and pick a random hash function $h :
U\to [m]$.
%and let the least significant $\log_2(k)$ bits denote the
%\emph{bin} of a key and the remaining bits denote the \emph{value}.
for each
element $x$ split $h(x)$ into two parts $b(x), v(x)$,
where $b(x) : U\to [k]$ is given by $h(x)\bmod k$ and $v(x)$ is given by
$\floor{h(x)/k}$. To create the sketch $S_{OPH}(A)$ of
size $k$ we simply let $S_{OPH}(A)[i] = \min_{a\in A, b(a) = i} v(a)$.
To estimate the similarity of two sets $A$ and $B$ we simply take the fraction
of indices, $i$, where $S_{OPH}(A)[i] = S_{OPH}(B)[i]$.

This is, however, not an unbiased estimator, as there may be \emph{empty bins}.
Thus, \cite{Shrivastava14oneperm,Shrivastava14densify} worked on handling empty
bins. They showed that the following addition gives an unbiased estimator with good
variance. For each index $i\in [k]$ let $b_i$ be a random bit. Now, for a given
sketch $S_{OPH}(A)$, if the $i$th bin is empty we copy the value of the closest
non-empty bin going left (circularly) if $b_i=0$ and going right if $b_i=1$. We
also add $j\cdot C$ to this copied value, where $j$ is the distance
to the copied bin and $C$ is some sufficiently large offset parameter. The
entire construction is illustrated in Figure~\ref{fig:oph}
\begin{figure}[htbp]
    \centering
    \includegraphics[width=0.45\textwidth]{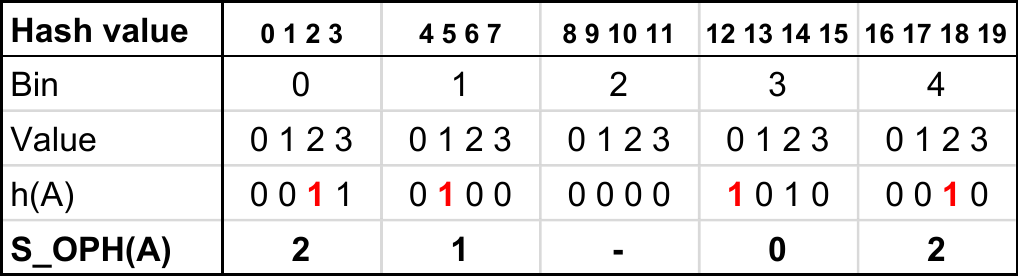}
    \hspace{1em}
    \includegraphics[width=0.45\textwidth]{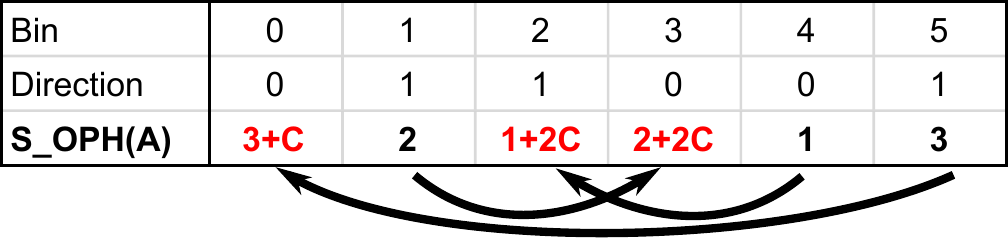}
    \caption{Left: Example of one permutation sketch creation of a set $A$ with
    $|U|=20$ and $k=5$. For each of the $20$ possible hash value the
    corresponding bin and value is displayed. The hash values of $A$, $h(A)$,
    are displayed as an indicator vector with the minimal value per bin marked
    in red. Note that the 3rd bin is empty. Right: Example of the densification
    from \cite{Shrivastava14densify} (right).}
    \label{fig:oph}
\end{figure}

\subsection{Feature hashing}\label{sec:fh}
Feature hashing (FH) introduced by Weinberger et al.~\cite{WeinbergerDLSA09} takes a
vector $v$ of dimension $d$ and produces a vector $v'$ of dimension $d' \ll d$
preserving (roughly) the norm of $v$. More precisely, let $h : [d]\to[d']$ and
$\sgn : [d] \to \{-1,+1\}$ be random hash functions, then $v'$ is defined as
$v'_i = \sum_{j, h(j) = i} \sgn(j)v_j$.
Weinberger et al.~\cite{WeinbergerDLSA09} (see also \cite{DasguptaKS10}) showed
exponential tail bounds on $\|v'\|_2^2$ when $\|v\|_\infty$ is sufficiently
small and $d'$ is sufficiently large.

\subsection{Locality-sensitive hashing}
The LSH framework of \cite{IndykM98} is a solution to the approximate near
neighbour search problem: Given a giant collection of sets $\mathcal{C} =
A_1,\ldots, A_n$, store a data structure such that, given a query set $A_q$, we
can, loosely speaking, efficiently find a $A_i$ with large $J(A_i,A_q)$.
Clearly, given the potential massive size of $\mathcal{C}$ it is infeasible to
perform a linear scan.

With LSH parameterized by positive integers $K,L$ we create a size
$K$ sketch $S_{oph}(A_i)$ (or using another method) for each $A_i\in
\mathcal{C}$. We then store the set $A_i$ in a large table indexed by this
sketch $T[S_{oph}(A_i)]$. For a given query $A_q$ we then go over all sets
stored in $T[S_{oph}(A_q)]$ returning only those that are ``sufficiently
similar''. By picking $K$ large enough we ensure that very distinct sets
(almost) never end up in the same bucket, and by repeating the data structure
$L$ independent times (creating $L$ such tables) we ensure that similar sets
are likely to be retrieved in at least one of the tables.

Recently, much work has gone into providing theoretically
optimal \cite{AndoniR15,AndoniLRW17,Christiani17} LSH. However, as noted in
\cite{AndoniILRS15}, these solutions require very sophisticated
locality-sensitive hash functions and are mainly impractical.
We therefore choose to focus on more practical variants relying either on OPH
\cite{Shrivastava14oneperm,Shrivastava14densify} or
FH~\cite{Charikar02,AndoniILRS15}.

\subsection{Mixed tabulation}\label{sec:mtab}
Mixed tabulation was introduced by~\cite{DahlgaardKRT15}.
For simplicity assume that we are hashing from the universe $[2^w]$
and fix integers $c,d$ such that $c$ is a divisor of $w$. Tabulation-based
hashing views each key $x$ as a list of $c$ characters $x_0,\ldots, x_{c-1}$,
where $x_i$ consists of the $i$th $w/c$ bits of $x$. We say that the alphabet
$\Sigma = [2^{w/c}]$. Mixed tabulation uses $x$ to derive $d$ additional
characters from $\Sigma$. To do this we choose $c$ tables $T_{1,i} : \Sigma\to
\Sigma^d$ uniformly at random and let $y = \oplus_{i=0}^c T_{1,i}[x_i]$
(here $\oplus$ denotes the XOR operation). The $d$ derived characters
are then $y_0,\ldots , y_{d-1}$. To create the final hash value we additionally
choose $c+d$ random tables $T_{2,i} : \Sigma\to [m]$ and define
\[
    h(x) = \bigoplus_{i\in [c]} T_{2,i}[x_i] \bigoplus_{i\in [d]}
    T_{2,i+c}[y_i]\ .
\]
Mixed Tabulation is extremely fast in practice due to the word-parallelism of
the XOR operation and the small table sizes which fit in fast cache.
It was proved in \cite{DahlgaardKRT15} that implementing OPH with mixed
tabulation gives Chernoff-style concentration bounds when estimating Jaccard
similarity.

\mt{Another advantage of mixed tabulation is when generating many
hash values for the same key. In this case, we can increase the output size of
the tables $T_{2,i}$, and then whp. over the choice of $T_{1,i}$ the resulting
output bits will be independent. As an example, assume that we want to map each
key to two 32-bit hash values. We then use a mixed tabulation hash function as
described above mapping keys to one 64-bit hash value, and then split this hash
value into two 32-bit values, which would be independent of each other with
high probability. Doing this with e.g. multiply-mod-prime hashing
would not work, as the output bits are not independent. Thereby we
significantly speed up the hashing time when generating many hash values for
the same keys.}

A sample implementation with $c=d=4$ and 32-bit keys and values can be found
below.
\begin{verbatim}
uint64_t mt_T1[256][4];  // Filled with random bits
uint32_t mt_T2[256][4];  // Filled with random bits

uint32_t mixedtab(uint32_t x) {
  uint64_t h=0;  // This will be the final hash value
  for(int i = 0;i < 4;++i, x >>= 8)
    h ^= mt_T1[(uint8_t)x][i];
  uint32_t drv=h >> 32;
  for(int i = 0;i < 4;++i, drv >>= 8)
    h ^= mt_T2[(uint8_t)drv][i];
  return (uint32_t)h;
}
\end{verbatim}

The main drawback to mixed tabulation hashing is that it needs a relatively
large random seed to fill out the tables $T_1$ and $T_2$. However, as noted in
\cite{DahlgaardKRT15} for all the applications we consider here it suffices to
fill in the tables using a $\Theta(\log |U|)$-independent hash
function.

\section{Feature Hashing with Mixed Tabulation}\label{sec:fhash_mixed}
As noted, Weinberger et al.~\cite{WeinbergerDLSA09} showed exponential tail
bounds for feature hashing. Here, we first prove improved concentration bounds,
and then, using techniques from \cite{DahlgaardKRT15} we argue that these
bounds still hold (up to a small additive factor polynomial in the universe
size) when implementing FH with mixed tabulation.

The concentration bounds we show are as follows (proved in
Appendix~\ref{app:proofs}).
%the full version).
\begin{theorem}\label{thm:fhash}
    Let $v\in \mathbb{R}^d$ with $\|v\|_2 = 1$ and let $v'$ be the
    $d'$-dimensional vector obtained by applying feature hashing implemented
    with truly random hash functions. Let $\eps,\delta\in (0,1)$.
    Assume that $d'\ge 16\eps^{-2}\lg(1/\delta)$ and $\|v\|_\infty \le
    \frac{\sqrt{\eps\log(1+\frac{4}{\eps})}}{6\sqrt{\log(1/\delta)\log(d'/\delta)}}$.
    Then it holds that
    \begin{align}
    	\label{eq:fhashguarantee}
        \Prp{1-\eps < \|v'\|_2^2 < 1+\eps} \ge 1 - 4\delta\,.
    \end{align}
\end{theorem}
Theorem \ref{thm:fhash} is very similar to the bounds on feature hashing
by Weinberger et al.~\cite{WeinbergerDLSA09} and Dasgupta et al.~\cite{DasguptaKS10}, but
improves on the requirement on the size of $\norm{v}_\infty$.
Weinberger et al.~\cite{WeinbergerDLSA09} show that \eqref{eq:fhashguarantee}
holds if $\norm{v}_\infty$ is bounded by $\frac{\eps}{18\sqrt{\log(1/\delta)\log(d'/\delta)}}$,
and Dasgupta et al.~\cite{DasguptaKS10} show that \eqref{eq:fhashguarantee} holds if $\norm{v}_{\infty}$
is bounded by $\sqrt{\frac{\eps}{16\log(1/\delta)\log^2(d'/\delta)}}$.
We improve on these results factors of $\Theta\!\left(\sqrt{\frac{1}{\eps}\log(1/\eps)}\right)$
and $\Theta\!\left(\sqrt{\log(1/\eps)\log(d'/\delta)}\right)$ respectively.
We note that if we use feature hashing with a pre-conditioner (as in e.g. \cite[Theorem 1]{DasguptaKS10})
these improvements translate into an improved running time.

Using \cite[Theorem 1]{DahlgaardKRT15} we get the following corollary.
\begin{corollary}\label{cor:fhash_mixed}
    Let $v,\eps,\delta$ and $d'$ be as in Theorem \ref{thm:fhash}, and let $v'$ be the
    $d'$-dimensional vector obtained using feature hashing on $v$ implemented
    with mixed tabulation hashing. Then, if $\mathrm{supp}(v)\le
    |\Sigma|/(1+\Omega(1))$ it holds that
    \[
        \Prp{1-\eps < \|v'\|_2^2 < 1+\eps} \ge 1 - 4\delta -
        O \left ( |\Sigma|^{1-\floor{d/2}} \right ) \,.
    \]
\end{corollary}
In fact Corollary~\ref{cor:fhash_mixed} holds even if
both $h$ and $\sgn$ from Section~\ref{sec:fh} are implemented using the same
hash function. I.e., if \mt{$h^\star : [d]\to \{-1,+1\}\times[d']$ is a mixed
tabulation hash function as described in Section~\ref{sec:mtab}.}

We note that feature hashing is often applied on very high dimensional, but
sparse, data (e.g. in \cite{AndoniILRS15}), and thus the requirement
$\mathrm{supp}(v)\le |\Sigma|/(1+\Omega(1))$ is not very prohibitive.
Furthermore, the target dimension $d'$ is usually logarithmic in the universe,
and then Corollary~\ref{cor:fhash_mixed} still works for vectors with
polynomial support giving an exponential decrease.

\section{Experimental evaluation}
\label{sec:experiments}

We experimentally evaluate several different basic hash functions. We first
perform an evaluation of running time. We then evaluate the fastest hash
functions on synthetic data confirming the theoretical results of
Section~\ref{sec:fhash_mixed} and \cite{DahlgaardKRT15}. Finally, we demonstrate that
even on real-world data, the provable guarantees of mixed tabulation sometimes
yields systematically better results.

Due to space restrictions, we only present some of our experiments here. The
rest are included in Appendix~\ref{app:figures}.
%, and refer to the full version for more details.

We consider some of the most popular and fast hash functions employed in practice
in $k$-wise PolyHash \cite{CarterW79}, Multiply-shift
\cite{dietzfel97closest}, MurmurHash3 \cite{Appleby16}, CityHash
\cite{PikeA11}, and the cryptographic hash function Blake2
\cite{AumassonNWW13}.
Of these hash functions only mixed tabulation (and very high degree PolyHash)
\emph{provably} works well for the applications we consider. However,
Blake2 is a cryptographic function which provides similar guarantees
conditioned on certain cryptographic assumptions being true.
The remaining hash functions have provable weaknesses, but often
work well (and are widely employed) in practice. See e.g.~\cite{murmurbreak}
who showed how to break both MurmurHash3 and Cityhash64.

All experiments are implemented in C++11 using a random seed from
\url{http://www.random.org}. The seed
for mixed tabulation was filled out using a random $20$-wise PolyHash
function. All keys and hash outputs were 32-bit integers to ensure efficient
implementation of multiply-shift and PolyHash using Mersenne prime $p =
2^{61}-1$ and GCC's 128-bit integers.

We perform two time experiments, the results of which are presented in
Table~\ref{tab:times}. Namely, we evaluate each hash function on the same
$10^7$ randomly chosen integers and use each hash function to implement FH on
the News20 dataset (discussed later).
We see that the only two functions faster than mixed tabulation
are the very simple multiply-shift and 2-wise PolyHash.
MurmurHash3 and CityHash were roughly  $30$-$70$\% slower than mixed
tabulation. This even though we used the official
implementations of MurmurHash3, CityHash and Blake2 which are highly optimized
to the x86 and x64 architectures, whereas mixed tabulation is just
standard, portable C++11 code. The cryptographic hash function, Blake2, is
orders of magnitude slower as we would expect.

\begin{table}[htbp]
    \caption{Time taken to evaluate different hash functions to 1) hash $10^7$
    random numbers, and 2) perform feature hashing with $d'=128$ on
    the entire News20 data set.}
    \label{tab:times}
    \centering
    \begin{tabular}{l S[table-format=4.2] S[table-format=4.2]}
        \toprule
        Hash function & {time ($1..10^7$)} & {time (News20)} \\
        \midrule
        Multiply-shift & 7.72{\,ms} & 55.78{\,ms} \\
        2-wise PolyHash & 17.55{\,ms} & 82.47{\,ms} \\
        3-wise PolyHash & 42.42{\,ms} & 120.19{\,ms}\\
        MurmurHash3 & 59.70{\,ms} & 159.44{\,ms} \\
        CityHash & 59.06{\,ms} & 162.04{\,ms} \\
        Blake2 & 3476.31{\,ms} & 6408.40{\,ms} \\
        \midrule
        Mixed tabulation & 42.98{\,ms} & 90.55{\,ms} \\
        \bottomrule
    \end{tabular}
\end{table}

Based on Table~\ref{tab:times} we choose to compare mixed tabulation to
multiply-shift, 2-wise PolyHash and MurmurHash3. We also include results for
20-wise PolyHash as a (cheating) way to ``simulate'' truly random hashing.

\subsection{Synthetic data}\label{sec:synth}
For a parameter, $n$, we generate two sets $A,B$ as follows. The intersection
$A\cap B$ is created by sampling each integer from $[2n]$
independently at random with probability $1/2$. The symmetric difference is
generated by sampling $n$ numbers greater than $2n$ (distributed evenly to $A$
and $B$). Intuitively, with a hash function like $(ax+b)\bmod p$, the dense
subset of $[2n]$ will be mapped very systematically and is likely (i.e.
depending on the choice of $a$) to be spread out evenly. When using OPH, this means that
elements from the intersection is more likely to be the smallest element in
each bucket, leading to an over-estimation of $J(A,B)$.

We use OPH with densification as in \cite{Shrivastava14densify} implemented
with different basic hash functions to estimate $J(A,B)$. We generate one
instance of $A$ and $B$ and perform $2000$ independent repetitions for each
different hash function on these $A$ and $B$. Figure~\ref{fig:synth_oph} shows
the histogram and mean squared error (MSE) of estimates obtained with $n=2000$
and $k=200$. The figure confirms
the theory: Both multiply-shift and 2-wise PolyHash exhibit bias and bad
concentration whereas both mixed tabulation and MurmurHash3 behaves essentially
as truly random hashing. We also performed experiments with $k=100$ and
$k=500$ and considered the case of $n = k/2$, where we expect many empty
bins and the densification of \cite{Shrivastava14densify} kicks in. All
experiments obtained similar results as
Figure~\ref{fig:synth_oph}.
\begin{figure}[htbp]
    \centering
    \includegraphics[width = \textwidth]{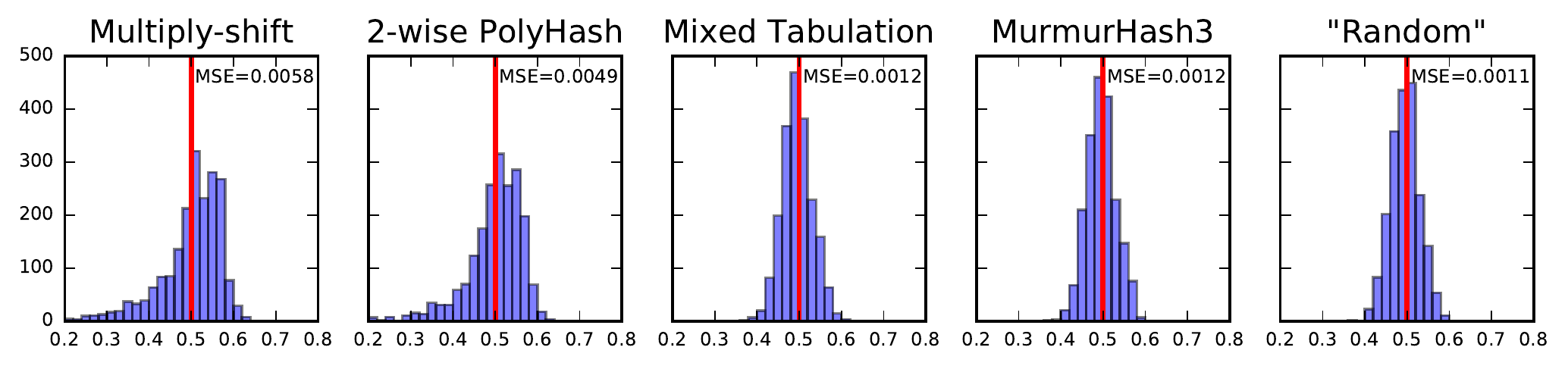}
    \caption{Histograms of set similarity estimates obtained using OPH with
    densification of \cite{Shrivastava14densify} on synthetic data implemented
    with different basic hash families and $k=200$.
    The mean squared error for each hash function is displayed in the
    top right corner.}
    \label{fig:synth_oph}
\end{figure}

For FH we obtained a vector $v$ by taking the indicator vector of
a set $A$ generated as above and normalizing the length. For each hash function
we perform $2000$ independent repetitions of the following experiment: Generate
$v'$ using FH and calculate $\|v'\|_2^2$. Using a good hash
function we should get good concentration of this value around $1$.
Figure~\ref{fig:synth_fh} displays the histograms and MSE we obtained for $d'=200$.
Again we see that multiply-shift and 2-wise PolyHash give poorly concentrated
results, and while the results are not biased this is only because of a very
heavy tail of large values. We also ran experiments with $d'=100$ and $d'=500$
which were similar.
\begin{figure}[htbp]
    \centering
    \includegraphics[width = \textwidth]{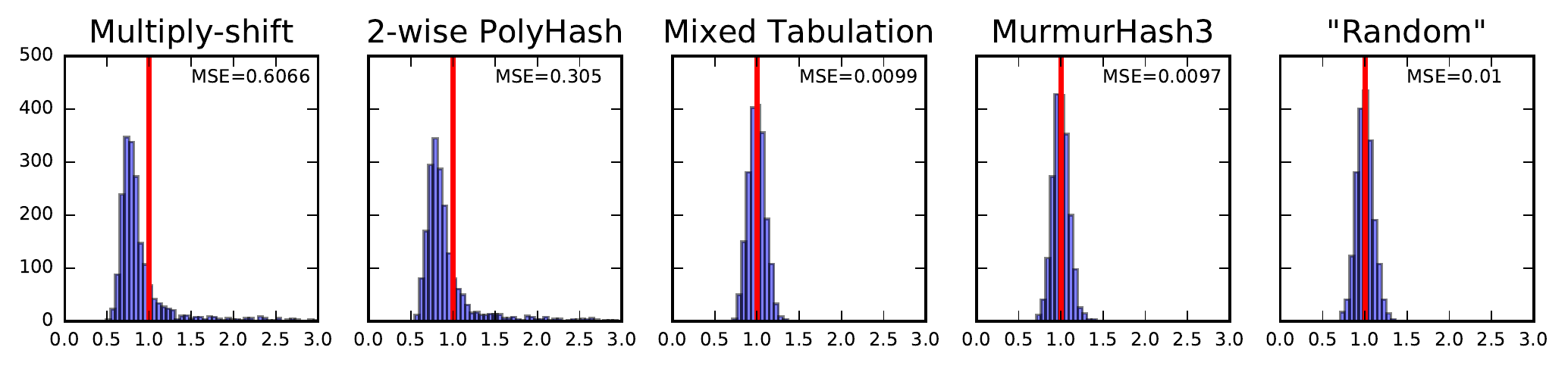}
    \caption{Histograms of the 2-norm of the vectors output by FH on synthetic
    data implemented with different basic hash families and $d'=200$.
    The mean squared error for each hash function is displayed in the
    top right corner.}
    \label{fig:synth_fh}
\end{figure}

We briefly argue that this input is in fact quite natural: When encoding a
document as shingles or bag-of-words, it is quite common to let frequent
words/shingles have the lowest identifier (using fewest bits). In this case the
intersection of two sets $A$ and $B$ will likely be a dense subset of small
identifiers. This is also the case when using Huffman Encoding \cite{huf52}, or
if identifiers are generated on-the-fly as words occur. Furthermore, for images
it is often true that a pixel is more likely to have a non-zero value if its
neighbouring pixels have non-zero values giving many consecutive non-zeros.

\paragraph{Additional synthetic results}
We also considered the following synthetic dataset, which actually showed even
more biased and poorly concentrated results. For similarity estimation we used
elements from $[4n]$, and let the symmetric difference be uniformly random
sampled elements from $\{0\ldots, n-1\}\cup\{3n,\ldots,4n-1\}$ with probability
$1/2$ and the intersection be the same but for $\{n,\ldots, 3n-1\}$. This gave
an MSE that was rougly $6$ times larger for multiply-shift and $4$ times larger
for $2$-wise PolyHash compared to the other three. For
feature hashing we sampled the numbers from $0$ to $3n-1$ independently at
random with probability $1/2$ giving an MSE that was $20$ times higher for
multiply-shift and $10$ times higher for 2-wise PolyHash.

We also considered both datasets without the sampling, which showed an even
wider gap between the hash functions.

\subsection{Real-world data}

We consider the following real-world data sets
\begin{itemize}
    \item \textbf{MNIST} \cite{mnist} Standard collection of handwritten
        digits. The average number of non-zeros is roughly 150 and the total
        number of features is 728.
        We use the standard partition of 60000
        database points and 10000 query points.
    \item \textbf{News20} \cite{ChangL11} Collection of newsgroup documents.
        The average number of non-zeros is roughly 500 and the total number of
        features is roughly $1.3\cdot 10^6$.
        We randomly split the set into two
        sets of roughly 10000 database and query points.
\end{itemize}
These two data sets cover both the sparse and dense regime, as well as the
cases where each data point is similar to many other points or few other
points. For MNIST this number is roughly 3437 on average and for News20 it is
roughly 0.2 on average for similarity threshold above $1/2$.

\paragraph{Feature hashing}
We perform the same experiment as for synthetic data by calculating
$\|v'\|_2^2$ for each $v$ in the data set with $100$ independent repetitions of
each hash function (i.e. getting $6,000,000$ estimates for MNIST). Our results
are shown in Figure~\ref{fig:real_fh} for output dimension $d' = 128$. Results
with $d'=64$ and $d'=256$ were similar.
\begin{figure}[htbp]
    \centering
    \includegraphics[width = \textwidth]{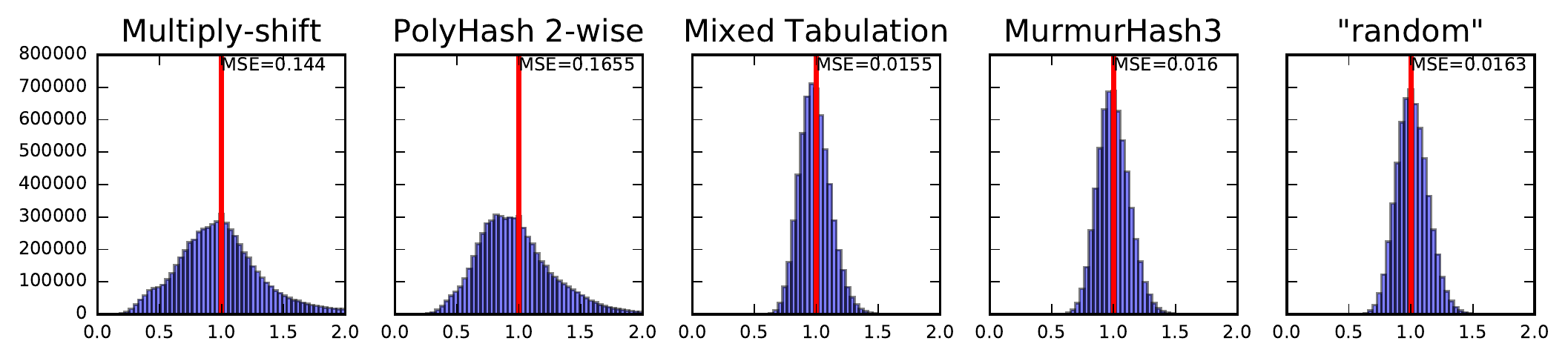}
    
    \includegraphics[width = \textwidth]{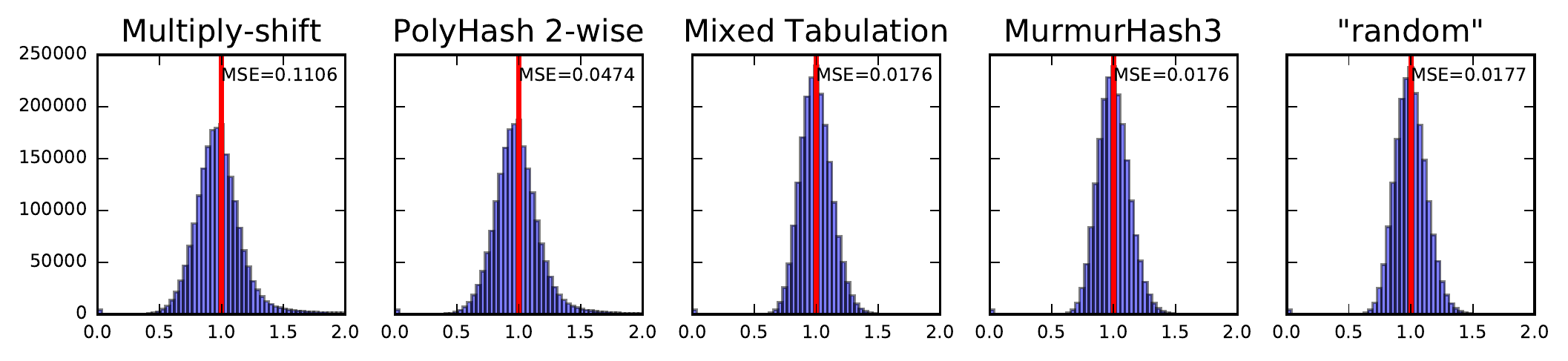}
    \caption{Histograms of the norm of vectors output by FH on the
    MNIST (top) and News20 (bottom) data sets implemented with different basic
    hash families and $d'=128$.
    The mean squared error for each hash function is displayed in the
    top right corner.}
    \label{fig:real_fh}
\end{figure}
The results confirm the theory and show that mixed tabulation performs
essentially as well as a truly random hash function
clearly outperforming the weaker hash functions, which produce poorly
concentrated results. This is particularly clear for the MNIST data set, but
also for the News20 dataset, where e.g. 2-wise Polyhash resulted in
$\|v'\|_2^2$ as large as $16.671$ compared to $2.077$ with mixed tabulation.

\paragraph{Similarity search with LSH}
We perform a rigorous evaluation based on the setup of
\cite{Shrivastava14oneperm}. We test all combinations of $K\in\{8,10,12\}$ and
$L\in\{8,10,12\}$. For readability we only provide results for multiply-shift
and mixed tabulation and note that the results obtained for 2-wise PolyHash and
MurmurHash3 are essentially identical to those
for multiply-shift and mixed tabulation respectively.

Following \cite{Shrivastava14oneperm} we evaluate the
results based on two metrics: 1) The fraction of total data points retrieved
per query, and 2) the \emph{recall} at a given threshold $T_0$ defined as the
ratio of retrieved data points having similarity at least $T_0$ with the query
to the total number of data points having similarity at least $T_0$ with
the query. Since the recall may be inflated by poor hash functions that just
retrieve many data points, we instead report \#retrieved/recall-ratio, i.e. the
number of data points that were retrieved divided by the percentage of recalled
data points. The goal is to minimize this ratio as we want to simultaneously
retrieve few points and obtain high recall. Due to space restrictions we
only report our results for $K=L=10$.
We note that the other results were similar.

Our results can be seen in Figure~\ref{fig:lsh}. The results somewhat echo what
we found on synthetic data. Namely, 1) Using multiply-shift overestimates the
similarities of sets thus retrieving more points, and 2) Multiply-shift gives
very poorly concentrated results. As a consequence of 1) Multiply-shift does,
however, achieve slightly higher recall (not visible in the figure), but
despite recalling slightly more points, the \#retrieved / recall-ratio of
multiply-shift is systematically worse.

\begin{figure}[htbp]
    \begin{center}
        \includegraphics[height=.24\textheight]{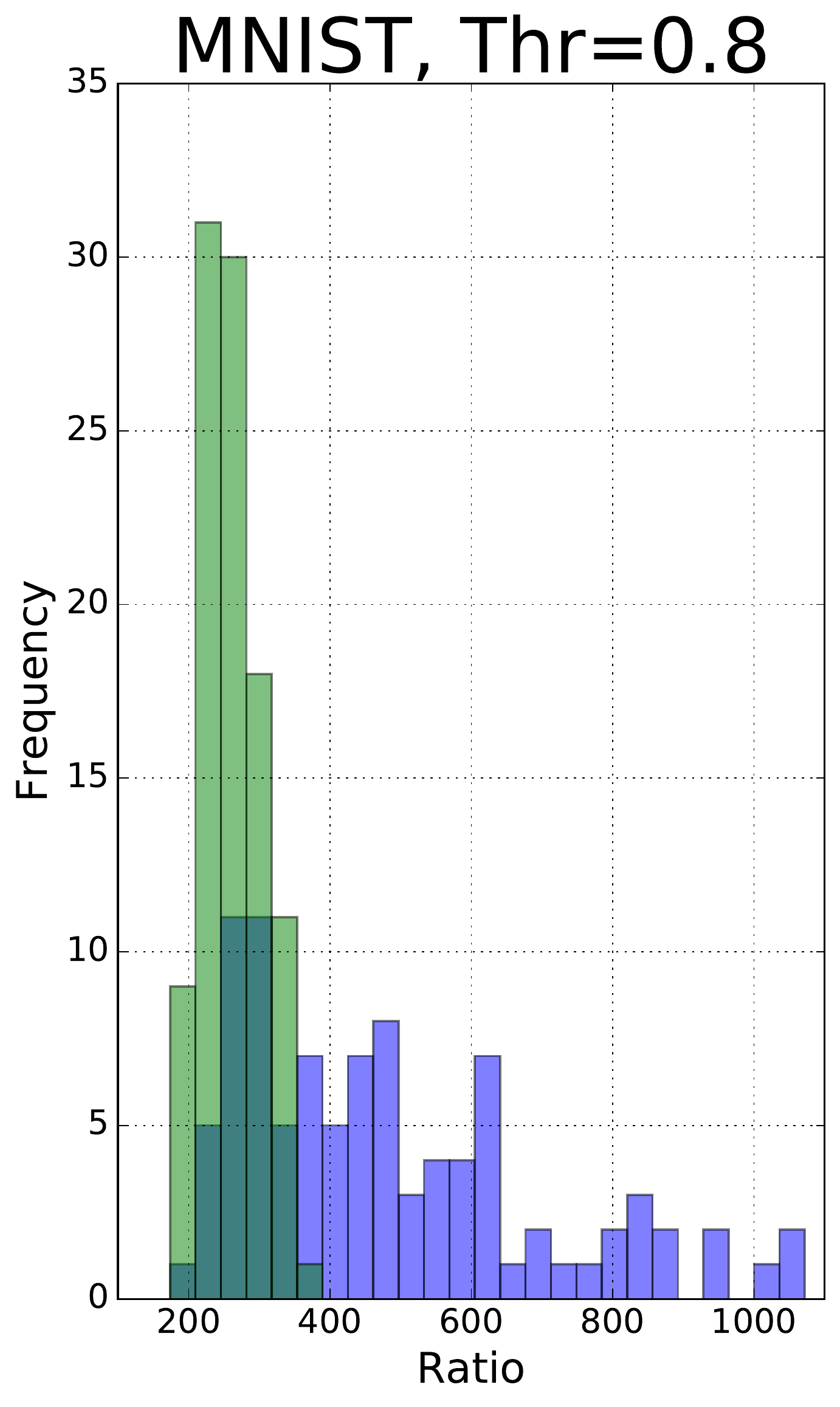}
        \includegraphics[height=.24\textheight]{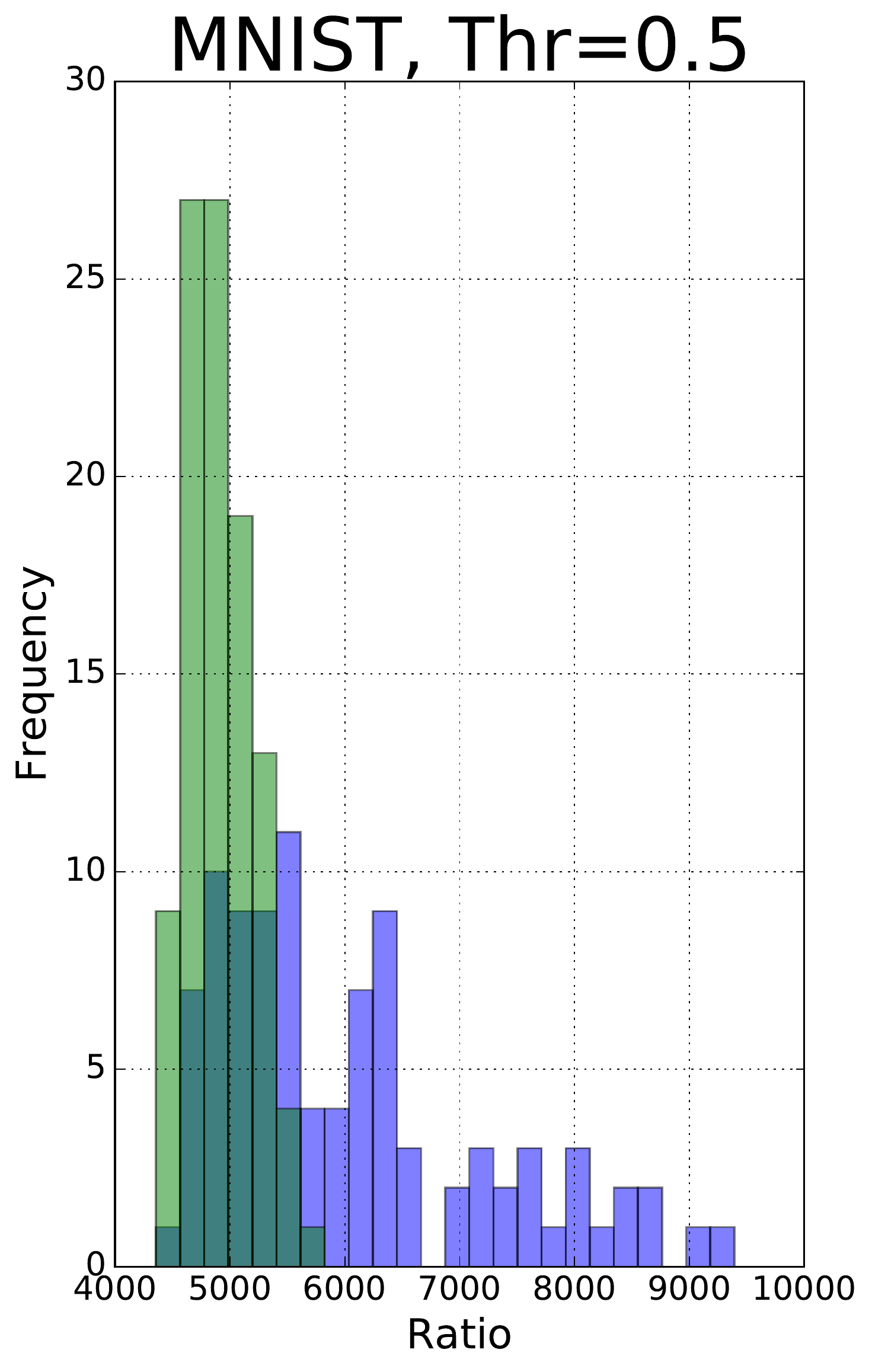}
        \includegraphics[height=.24\textheight]{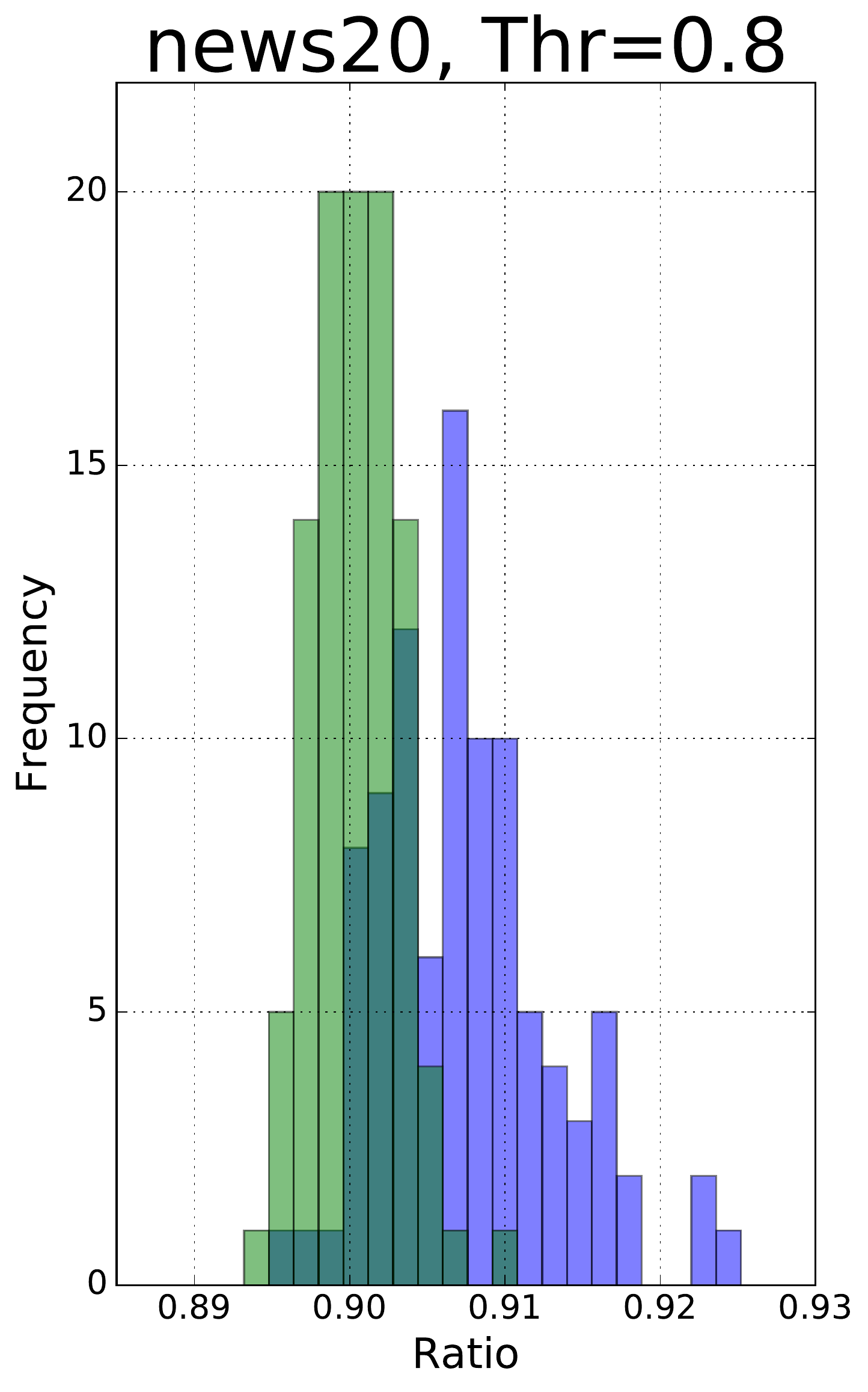}
        \includegraphics[height=.24\textheight]{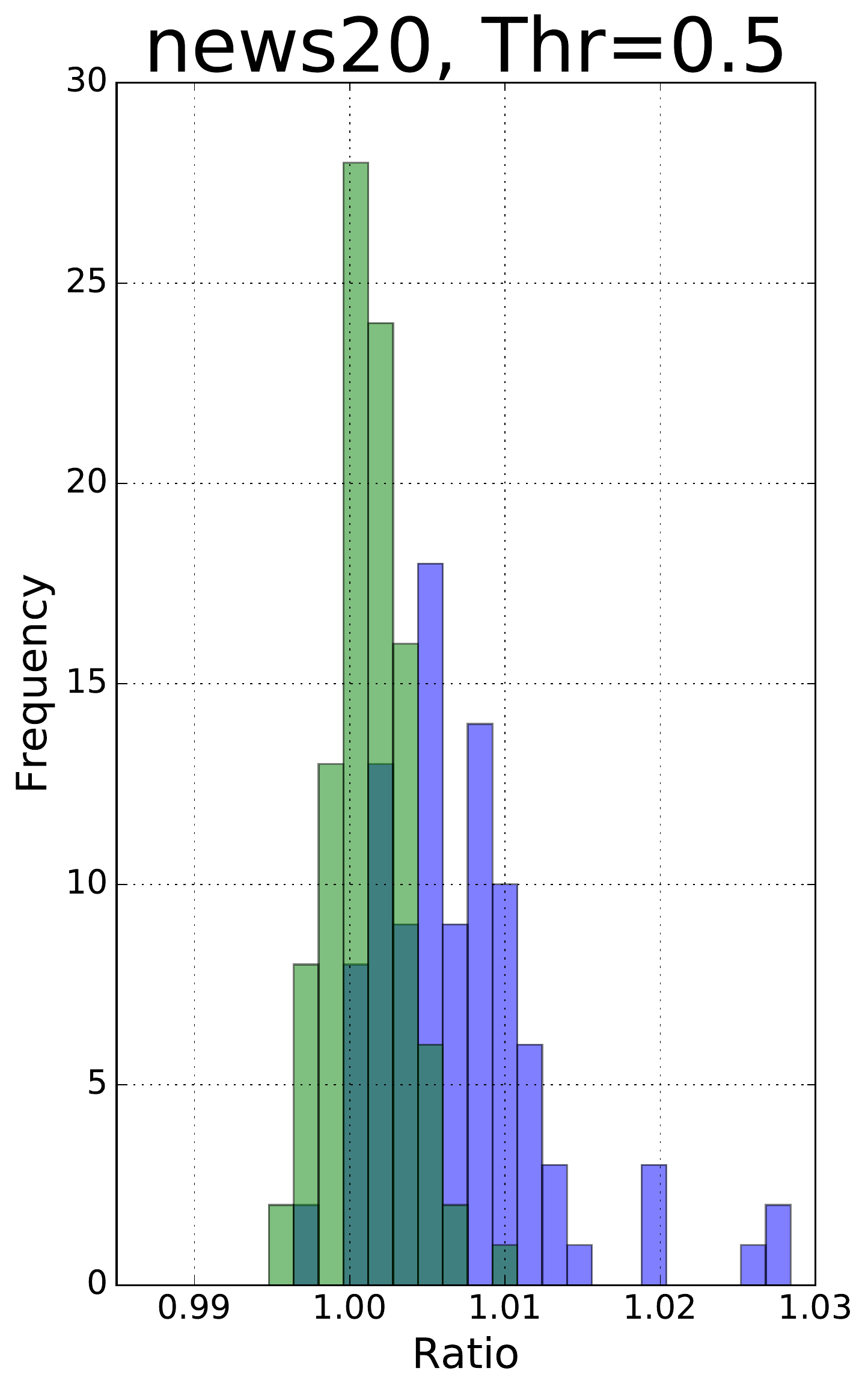}

        \caption{Experimental evaluation of LSH with OPH and different hash
        functions with $K=L=10$. The hash functions used are multiply-shift
        (blue) and mixed tabulation (green). The value studied is the
        retrieved / recall-ratio (lower is better).}
        \label{fig:lsh}
    \end{center}
\end{figure}

\section{Conclusion}
In this paper we consider mixed tabulation for computational primitives in
computer vision, information retrieval, and machine learning. Namely,
similarity estimation and feature hashing. It was previously
shown~\cite{DahlgaardKRT15} that mixed tabulation provably works essentially as
well as truly random for similarity estimation with one permutation hashing. We
complement this with a similar result for FH when the input vectors are sparse,
even improving on the concentration bounds for truly random hashing found by
\cite{WeinbergerDLSA09,DasguptaKS10}.

Our empirical results demonstrate this in practice. Mixed tabulation
significantly outperforms the simple hashing schemes and is not much slower.
Meanwhile, mixed tabulation is 40\% faster than both MurmurHash3 and CityHash,
which showed similar performance as mixed tabulation. However, these two hash
functions do not have the same theoretical guarantees as mixed tabulation. We
believe that our findings make mixed tabulation the best candidate for
implementing these applications in practice.

\section*{Acknowledgements} 
The authors gratefully acknowledge support from Mikkel Thorup's Advanced Grant
\emph{DFF-0602-02499B} from the \emph{Danish Council for Independent Research}
as well as the \emph{DABAI project}.
Mathias Bæk Tejs Knudsen gratefully acknowledges support from the \emph{FNU
project AlgoDisc}.

\bibliographystyle{plainurl}
\bibliography{mlstuff}

\appendix
\newpage
\centerline{\Huge Appendix}
\section{Omitted proofs}\label{app:proofs}
Below we include the proofs that were omitted due to space constraints.

We are going to use the following corollary of \cite[Theorem 2.1]{fanGL2012hoeffding}.
\begin{corollary}[\cite{fanGL2012hoeffding}]
	\label{cor:generalChernoff}
	Let $X_1,X_2,\ldots,X_n$ be random variables with $\Ep{X_i \mid X_1,\ldots,X_{i-1}} = 0$
	for each $i = 1,2,\ldots,n$ and let $X_{\le k}$ and $V_k$ be defined by:
	\begin{align*}
		X_{\le k} = \sum_{i = 1}^k X_i
		, \, 
		V_k = \sum_{i=1}^k \Ep{X_i^2 \mid X_1,\ldots,X_{i-1}}
		\, .
	\end{align*}
	
	For $c, \sigma^2, M > 0$ it holds that:
	\begin{align*}
		\Prp{
			\exists k: \abs{X_{\le k}} \ge c\sigma^2, \,
			V_k \le \sigma^2
		}
		\le 
		2 \exp \left (
			- \frac{1}{2} c \log (1 + cM)
			\sigma^2 / M
		\right )
		+
		\Prp{\max_k \abs{X_k} > M}
		\, .
	\end{align*}
\end{corollary}
\begin{proof}
Let $X_i' = X_i$ if $\abs{X_i} \le M$ and $X_i' = 0$ otherwise
Let $X_{\le k}' = \sum_{i=1}^k X_i'$ and $V_k' = \sum_{i=1}^k \Ep{X_i'^2 \mid X_1',\ldots,X_{i-1}'}$.
By Theorem 2.1 applied on $(X_i'/M)_{i \in \set{1,\ldots,n}}$
and Remark 2.1 in \cite{fanGL2012hoeffding} it holds that:
\begin{align*}
	\Prp{
		\exists k: X_{\le k}' \ge c\sigma^2, \,
		V_k' \le \sigma^2
	}
	& =
	\Prp{
		\exists k: X_{\le k}'/M \ge c\sigma^2/M, \,
		V_k'/M \le \sigma^2/M^2
	}
	\\
	& \le 
	\left ( \frac{\sigma^2/M^2}{c\sigma^2/M + \sigma^2/M^2} \right )^{c\sigma^2/M + \sigma^2/M^2}
	e^{c\sigma^2/M}
	\\
	& \le 
	\exp \left (
		- \frac{1}{2} c \log (1 + cM)
		\sigma^2 / M
	\right )
	\, ,
\end{align*}
where the second inequality follows from a simple calculation.
By the same reasoning we get that same upper bound on the probability that
there exists $k$ such that $X_{\le k}' \le -c\sigma^2$ and $V_k' \le \sigma^2$.
So by a union bound:
\begin{align*}
	\Prp{
		\exists k: \abs{X_{\le k}'} \ge c\sigma^2, \,
		V_k' \le \sigma^2
	}
	\le 
	2\exp \left (
		- \frac{1}{2} c \log (1 + cM)
		\sigma^2 / M
	\right )
	\, .
\end{align*}
If there exists $k$ such that $\abs{X_{\le k}} \ge c\sigma^2$ and $V_k \le \sigma^2$,
then either $X_i = X_i'$ for each $i \in \set{1,2,\ldots,n}$ and therefore
$\abs{X_{\le k}'} \ge c\sigma^2$ and $V_k' \le \sigma^2$. Otherwise, there exists 
$i$ such that $X_i \neq X_i'$ and therefore $\max_i \set{\abs{X_i}} > M$. Hence
we get that:
\begin{align*}
	\Prp{
		\exists k: \abs{X_{\le k}'} \ge c\sigma^2, \,
		V_k' \le \sigma^2
	}
	& \le 
	\Prp{
		\exists k: \abs{X_{\le k}} \ge c\sigma^2, \,
		V_k \le \sigma^2
	}
	+
	\Prp{\max_k \abs{X_k} > M}
	\\
	& \le 
	2\exp \left (
		- \frac{1}{2} c \log (1 + cM)
		\sigma^2 / M
	\right )
	+
	\Prp{\max_k \abs{X_k} > M}
	\, .
\end{align*}
\end{proof}

Below follows the proof of Theorem~\ref{thm:fhash}. We restate it here as
Theorem~\ref{thm:featurehashingimprovement} with slightly different notation.
\begin{theorem}
	\label{thm:featurehashingimprovement}
	Let $d,d'$ be dimensions, and $v \in \R^d$ a vector. Let $h : [d] \to [d']$
	and $\sgn: [d] \to \{-1,+1\}$ be uniformly random hash functions. Let
	$v' \in \R^{d'}$ be the vector defined by:
	\begin{align*}
		v_i' = \sum_{j, h(j) = i} \sgn(j) v_j\, .
	\end{align*}
	Let $\eps, \delta \in \left(0,1\right)$. If $d' \ge \alpha$,
	$\norm{v}_2 = 1$ and $\norm{v}_{\infty} \le \beta $,
	where $\alpha = 16 \eps^{-2}\lg(1/\delta)$ and $\beta =
	\frac{\sqrt{\eps \log\!\left(1+\frac{4}{\eps}\right)}}{6\sqrt{\log(1/\delta)\log(d'/\delta)}}$,
	then
	\begin{align*}
		\Prp{1-\eps < \norm{v'}_2^2 < 1+\eps} \ge 
		1 - 4\delta
		\, .
	\end{align*}
\end{theorem}
\begin{proof}[Proof of Theorem \ref{thm:featurehashingimprovement}]
Let $v'^{(k)}$ be the vector defined by:
\begin{align*}
	v'^{(k)}_i = \sum_{j \le k, h(j) = i} \sgn(j) v_j
	\, ,
	k \in \set{1,2,\ldots,d}
\end{align*}
We note that $v'^{(d)} = v'$. Let $X_k = \norm{v'^{(k)}}_2^2 - \norm{v'^{(k-1)}}_2^2 - v_k^2$.
A simple calculation gives that:
\begin{align*}
	\Ep{X_k \mid X_1,\ldots,X_{k-1}}
	= 0
	, \,
	\Ep{X_k^2 \mid X_1,\ldots,X_{k-1}}
	=
	\frac{4v_k^2}{d'} \cdot \norm{v'^{(k-1)}}_2^2
	\, .
\end{align*}
As in Corollary \ref{cor:generalChernoff} we define $X_{\le k} = X_1 + \ldots + X_k$
and $V_k = \sum_{i = 1}^k \Ep{X_k^2 \mid X_1,\ldots,X_{k-1}}$.
We see that $X_{\le d} = \norm{v'}_2^2 - \norm{v}_2^2 = \norm{v'}_2^2 - 1$.

Assume that $\abs{X_{\le d}} \ge \eps$, and let $k_0$ be the smallest integer such
that $X_{\le k_0} \ge \eps$. Then $\norm{v'^{(k)}}_2^2 \le 1+\eps$ for every
$k < k_0$, and therefore $V_{k_0} \le \frac{4(1+\eps)}{d'} \le \frac{8}{d'}$.
So we conclude that if $\abs{X_{\le d}} \ge \eps$ then there exists $k$ such that 
$\abs{X_{\le k}} \ge \eps$ and $V_{k} \le \frac{8}{d'}$.
Hence we get
\begin{align*}
	\Prp{\abs{\norm{v'}_2^2-1} \ge \eps}
	=
	\Prp{\abs{X_{\le d}} \ge \eps}
	\le 
	\Prp{\exists k: X_{\le k} \ge \eps, V_k \le \frac{8}{d'}}
	\, . 
\end{align*}
We now apply Corollary \ref{cor:generalChernoff} with
$\sigma^2 = \frac{8}{d'}$, $c = \frac{\eps d'}{8}$ and 
$M = \frac{8}{\eps \alpha}$ to obtain
\begin{align*}
	\Prp{\exists k: \abs{X_{\le k}} \ge \eps, V_k \le \frac{8}{d'}}
	& =
	\Prp{\exists k: \abs{X_{\le k}} \ge c \sigma^2, V_k \le \sigma^2}
	\\
	& \le 
	2 \exp \left (
		- \frac{1}{2} c \log \left ( 1 + cM \right )
		\sigma^2/M
	\right )
	+
	\Prp{\max_k \abs{X_k} > M}
	\\
	& =
	2 \exp \left (
		- \frac{\eps^2 \alpha}{16}
		\log \left ( 1 + \frac{d'}{\alpha} \right )
	\right )
	+
	\Prp{\max_k \abs{X_k} > \frac{8}{\eps \alpha}}
	\\
	& \le
	2 \exp \left (
		- \frac{\eps^2 \alpha}{16}
		\log(2)
	\right )
	+
	\Prp{\max_k \abs{X_k} > \frac{8}{\eps \alpha}}
	\\
	& =
	2\delta
	+
	\Prp{\max_k \abs{X_k} > \frac{8}{\eps \alpha}}
	\, .
\end{align*}
We see that $\abs{X_k} \le 2 \abs{v_k} \norm{v'^{(k-1)}}_\infty \le 2\norm{v}_\infty \norm{v'^{(k-1)}}_\infty$.
Therefore, by a union bound we get that
\begin{align*}
	\Prp{\max_k \abs{X_k} > \frac{8}{\eps \alpha}}
	\le 
	\Prp{\max_k \norm{v'^{(k)}}_\infty > \frac{4}{\eps \alpha \norm{v}_\infty}}
	\le 
	\sum_{i = 1}^{d'}
		\Prp{\max_k \abs{v'^{(k)}_i} > \frac{4}{\eps \alpha \norm{v}_\infty}}
	\, .
\end{align*}
Fix $i \in \set{1,2,\ldots,d'}$, and let $Z_k = v'^{(k)}_i - v'^{(k-1)}_i$ for $k=1,2,\ldots,d$.
Let $Z_{\le k} = Z_1 + \ldots + Z_k$ and
$W_k = \sum_{i = 1}^k \Ep{Z_k^2 \mid Z_1,\ldots,Z_{k-1}}$. Clearly we have that
$Z_{\le k} = v'^{(k)}_i$. Since the variables $Z_1,\ldots,Z_{d'}$ are independent
we get that
\begin{align*}
	\Ep{Z_k^2 \mid Z_1,\ldots,Z_{k-1}}
	=
	\Ep{Z_k^2}
	=
	\frac{1}{d'} v_k^2
	\, ,
\end{align*}
and in particular $W_k \le \frac{1}{d'}$.
Since $Z_k \in \set{-v_k,0,v_k}$ we always have that $\abs{Z_k} \le \norm{v}_\infty$.
We now apply Corollary \ref{cor:generalChernoff} with $c = \frac{4d'}{\eps \alpha \norm{v}_\infty}$,
$\sigma = \frac{1}{d'}$ and $M = \norm{v}_\infty$:
\begin{align*}
	\Prp{\max_k \abs{v'^{(k)}_i} > \frac{4}{\eps \alpha \norm{v}_\infty}}
	& =
	\Prp{\max_k \abs{Z_k} > c\sigma^2}
	\\
	& \le 
	2\exp\left (
		-\frac{1}{2} \frac{4d'}{\eps \alpha \norm{v}_\infty}
		\log\left(1 + \frac{4d'}{\eps \alpha}\right)
		\frac{1}{d' \norm{v}_\infty}
	\right )
	\\
	& \le
	2\exp\left (
		-\frac{2}{\eps \alpha \norm{v}^2_\infty}
		\log\left(1 + \frac{4}{\eps}\right)
	\right )
	\\
	& \le 
	2\exp\left (
		-\frac{2}{\eps \alpha \beta^2}
		\log\left(1 + \frac{4}{\eps}\right)
	\right )
	\le
	\frac{2\delta}{d'}
	\, .
\end{align*}
Hence we have concluded that 
\begin{align*}
	\Prp{\abs{\norm{v'}_2^2-1} \ge \eps}
	\le 
	2\delta + \sum_{i=1}^{d'} \frac{2\delta}{d'}
	=
	4\delta
	\, .
\end{align*}
Therefore,
\begin{align*}
	\Prp{1-\eps < \norm{v'}_2^2 < 1+\eps}
	=
	1 - \Prp{\abs{\norm{v'}_2^2-1} \ge \eps}
	\ge
	1 - 4 \delta
	\, ,
\end{align*}
as desired.
\end{proof}

\section{Additional experiments}\label{app:figures}
Here we include histograms for some of the additional experiments that were not
included in Section \ref{sec:experiments}.

\begin{figure}[htbp]
    \centering
    \includegraphics[width = \textwidth]{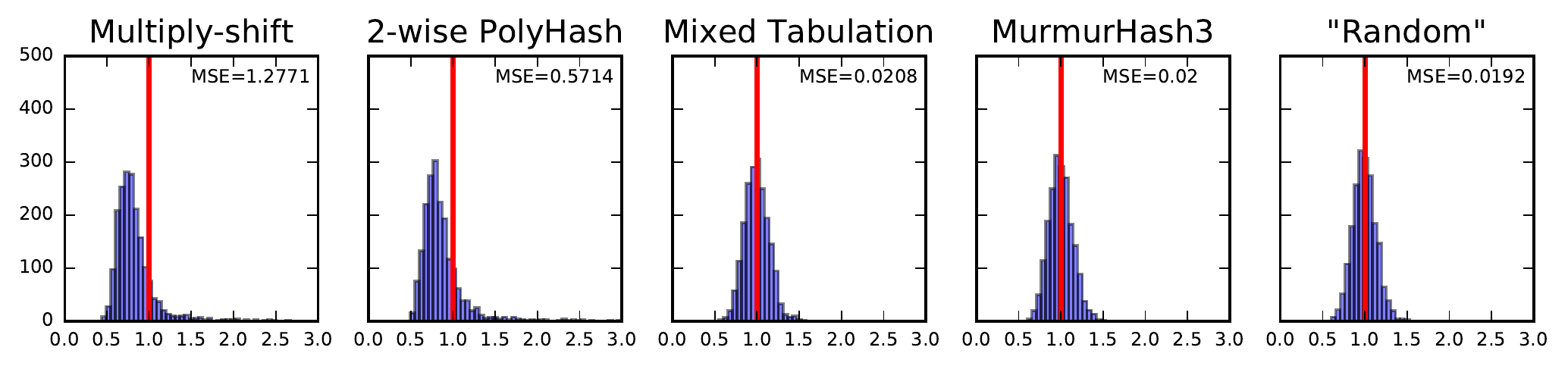}
    \includegraphics[width = \textwidth]{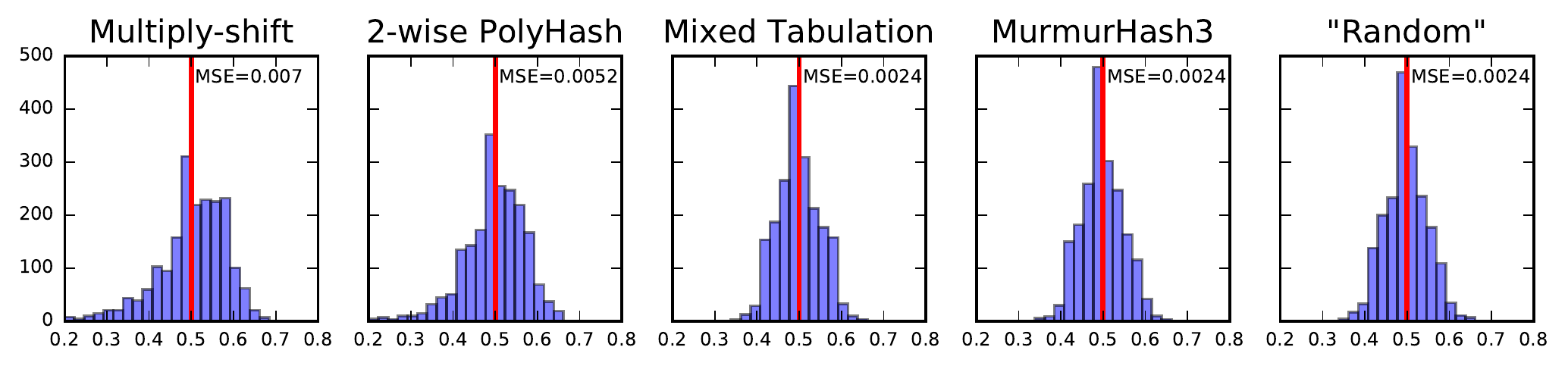}
    
    \caption{Similarity estimation with OPH (bottom) and FH (top) for $k=100$
    and $d'=100$ on the synthetic dataset of Section \ref{sec:synth}.}
\end{figure}

\begin{figure}[htbp]
    \centering
    \includegraphics[width = \textwidth]{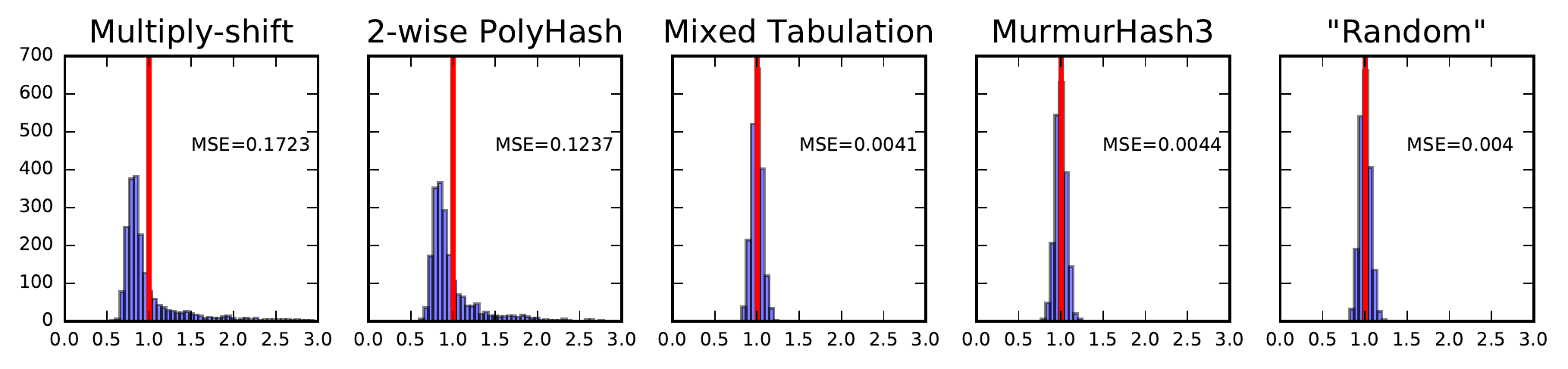}
    \includegraphics[width = \textwidth]{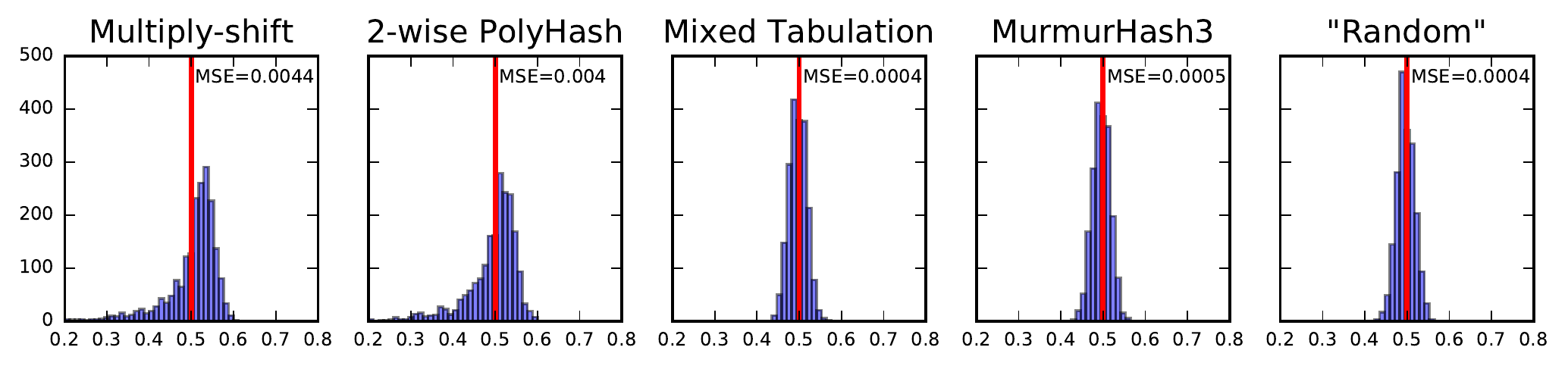}
    
    \caption{Similarity estimation with OPH (bottom) and FH (top) for $k=500$
    and $d'=500$ on the synthetic
    dataset of Section \ref{sec:synth}.}
\end{figure}

\begin{figure}[htbp]
    \centering
    \includegraphics[width = \textwidth]{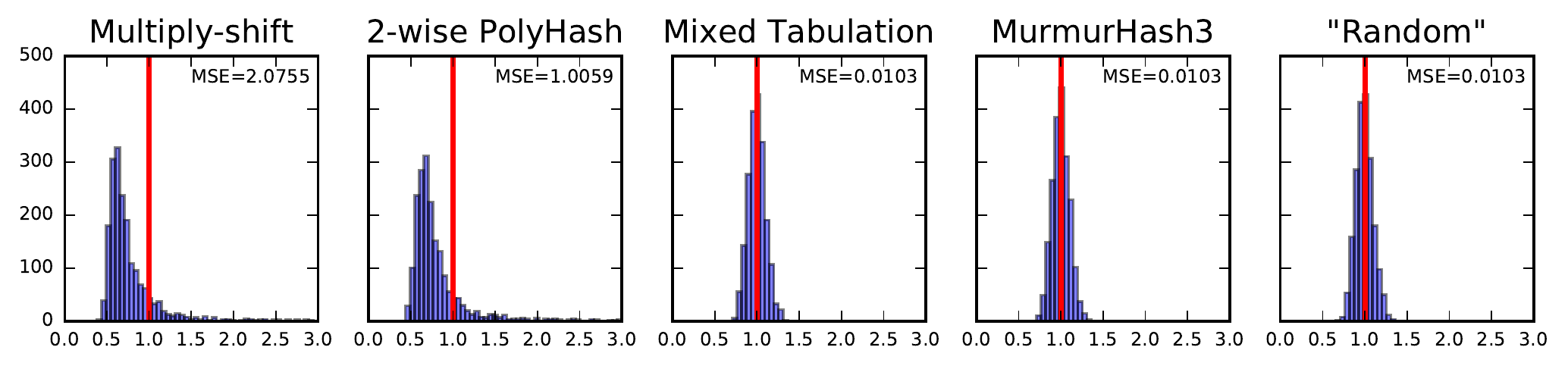}
    \includegraphics[width = \textwidth]{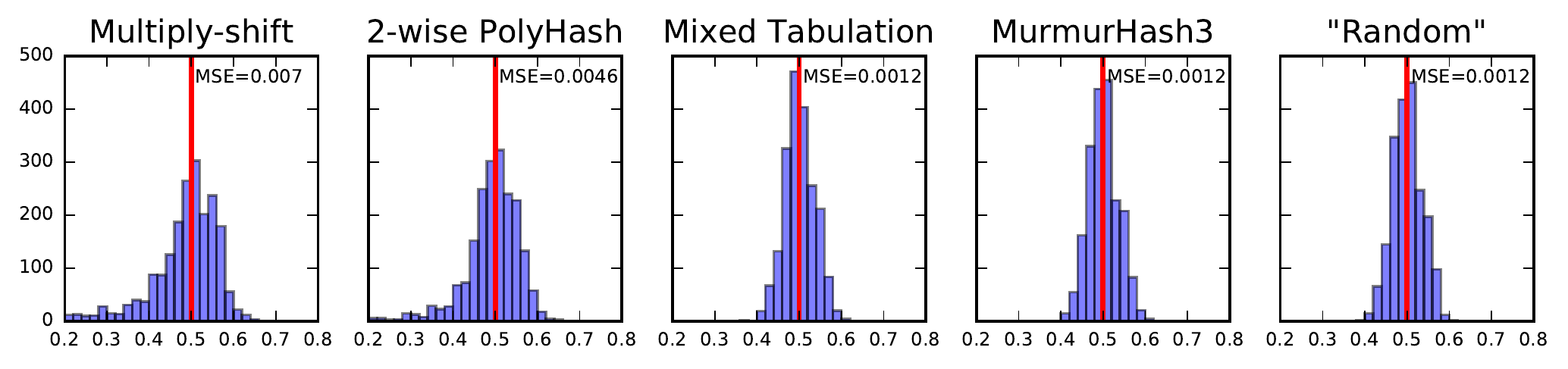}
    
    \caption{Similarity estimation with OPH (bottom) and FH (top) for $k=200$
    and $d'=200$ on the second synthetic
    dataset (with numbers from $[4n]$ with $n=2000$) of Section \ref{sec:synth}.}
\end{figure}

\begin{figure}[htbp]
    \centering
    \includegraphics[width = \textwidth]{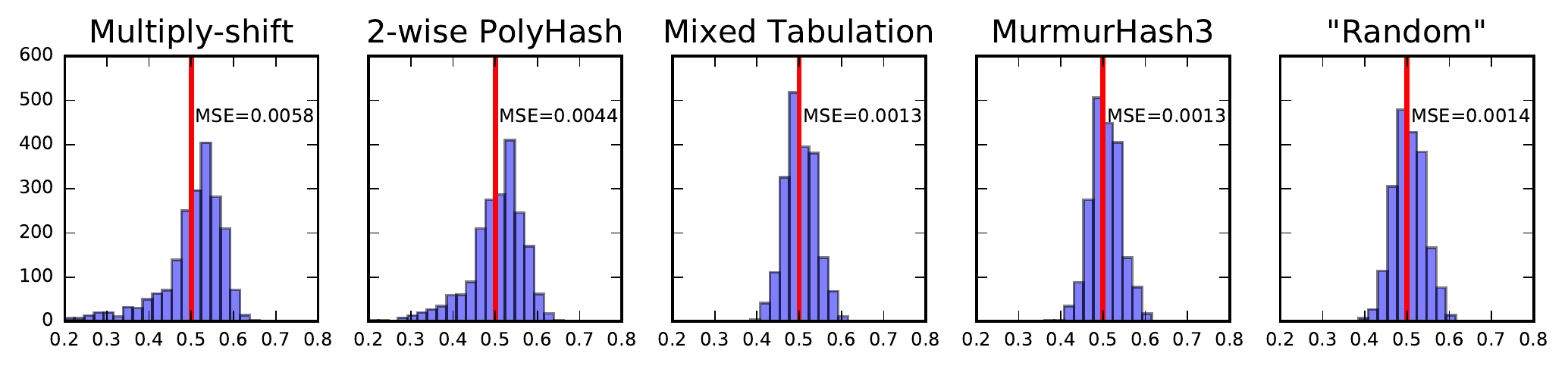}
    
    \caption{Similarity estimation with OPH for $k=200$ with sparse input
    vectors (size $\approx 150$) on the synthetic dataset of Section \ref{sec:synth}.}
\end{figure}

\begin{figure}[htbp]
    \centering
    \includegraphics[width = \textwidth]{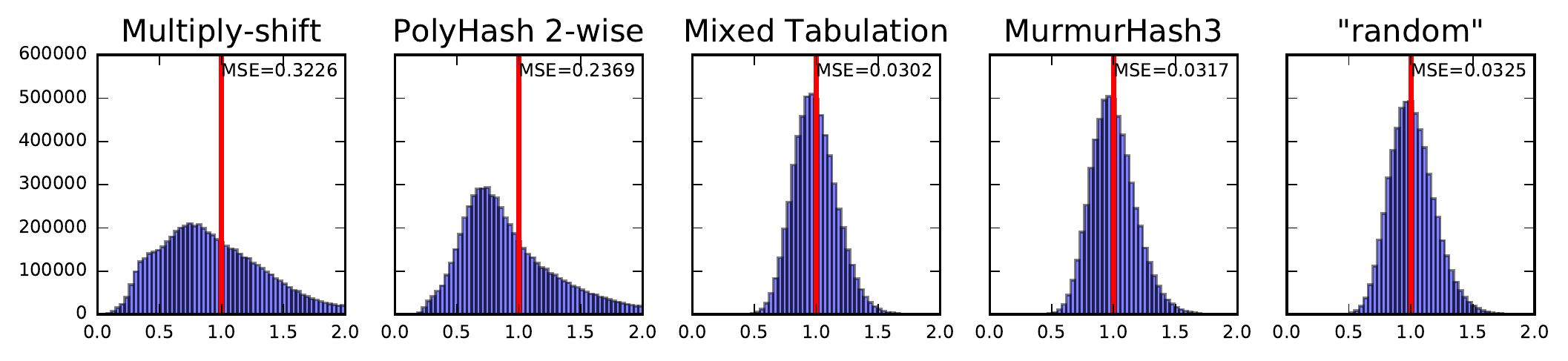}
    \includegraphics[width = \textwidth]{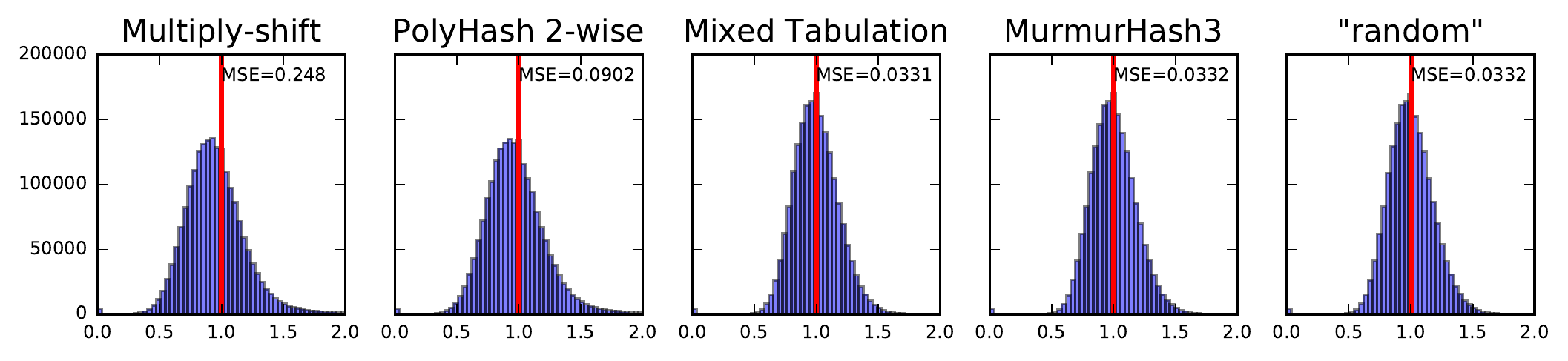}
    
    \caption{Norm of vector from FH with $d'=64$ for $100$ independent
    repetitions on MNIST (top) and News20 (bottom).}
\end{figure}

\begin{figure}[htbp]
    \centering
    \includegraphics[width = \textwidth]{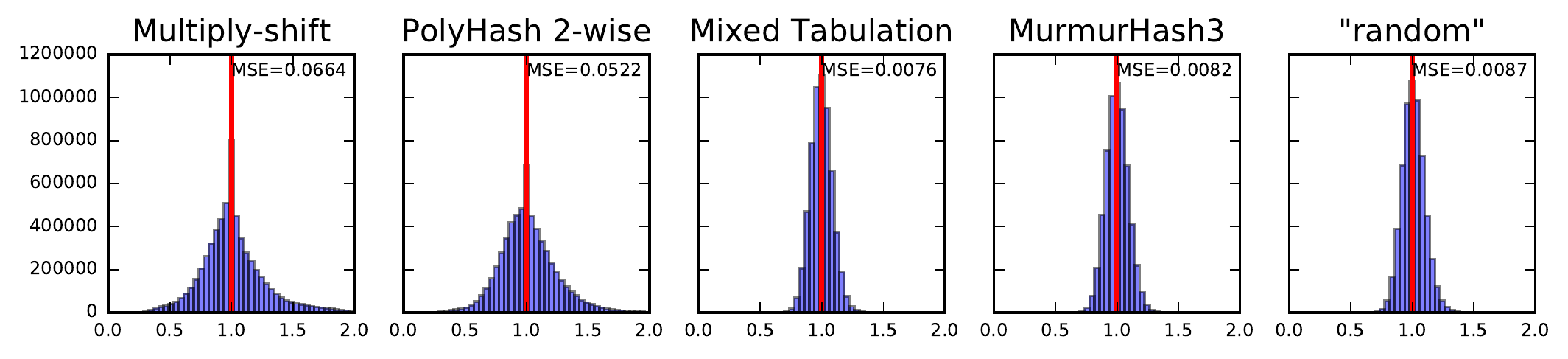}
    \includegraphics[width = \textwidth]{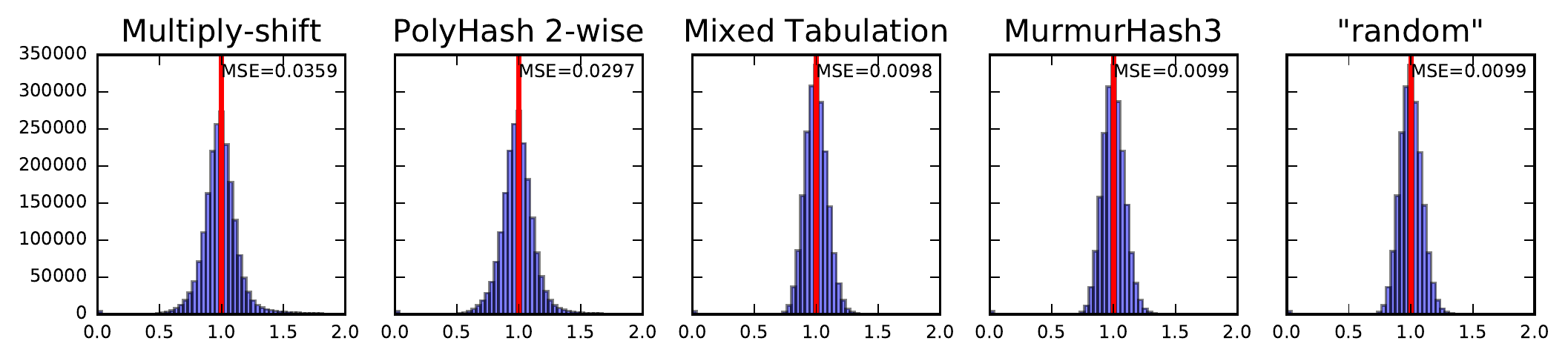}
    
    \caption{Norm of vector from FH with $d'=256$ for $100$ independent
    repetitions on MNIST (top) and News20 (bottom).}
\end{figure}

\end{document}